\documentclass[]{spie}  
\usepackage[]{graphicx}
\usepackage{setspace, dsfont}
\usepackage{amssymb,amsmath,mathrsfs,stmaryrd,upgreek,mathtools}
\usepackage{pgfplots, caption, tikz, subfigure, graphicx}
\usepackage{url}
\usepackage{bbm}
\usepgflibrary{arrows} 
\usepackage{varwidth}
\usetikzlibrary{spy}
\usepackage{graphicx,wrapfig,tikz}
\usetikzlibrary{calc,hobby} 
\title{Topology Reduction in \\ Deep Convolutional Feature Extraction Networks}

\newcommand{\N}{\mathbb{N}}

\newcommand{\R}{\mathbb{R}}

\newcommand{\la}{\lambda}

\newcommand\mydots{\ifmmode\ldots\else\makebox[1em][c]{.\hfil.\hfil.}\fi}

\newcommand{\Z}{\mathbb{Z}}

\DeclareMathOperator{\supp}{supp}

\author{Thomas Wiatowski\supit{a}, Philipp Grohs\supit{b}, and Helmut B\"olcskei\supit{a}
\skiplinehalf
\supit{a}Department of Information Technology and Electrical Engineering, ETH Zurich, Sternwartstrasse 7, 8092 Zurich, Switzerland \\
\supit{b}Faculty of  Mathematics, University of Vienna, Oskar Morgenstern Platz 1, \\1090 Vienna, Austria
}

\authorinfo{Further author information: (Send correspondence to T.W.)\\T.W.: E-mail: withomas@nari.ee.ethz.ch, Telephone: +41 44 63 22804\\  P.G.: E-mail: philipp.grohs@univie.ac.at, Telephone: +43 1 4277 55741\\  H.B.: E-mail: boelcskei@nari.ee.ethz.ch, Telephone: +41 44 63 23433}

\begin{document} 
  \maketitle 

\begin{abstract}
Deep convolutional neural networks (CNNs) used in practice employ potentially hundreds of layers and $10$,$000$s of nodes.  Such network sizes entail significant computational complexity due to the large number of convolutions that need to be carried out; in addition, a large number of parameters needs to be learned and stored. Very deep and wide CNNs may therefore not be well suited to applications operating under severe resource constraints as is the case, e.g., in low-power embedded and mobile platforms. This paper aims at understanding the impact of CNN topology, specifically depth and width, on the network's feature extraction capabilities. We address this question for the class of scattering networks that employ either Weyl-Heisenberg filters or wavelets, the modulus non-linearity, and no pooling. The exponential feature map energy decay results in Wiatowski et al., 2017, are generalized to $\mathcal{O}(a^{-N})$,  where an \emph{arbitrary} decay factor $a>1$ can be realized through  suitable choice of the Weyl-Heisenberg prototype function or the mother wavelet. We then show how networks of fixed (possibly small) depth $N$ can be designed to guarantee that $((1-\varepsilon)\cdot 100)\%$ of the input signal's energy are contained in the feature vector. Based on the notion of operationally significant nodes, we characterize, partly rigorously and partly heuristically, the topology-reducing effects of (effectively) band-limited input signals, band-limited filters, and  feature map symmetries.  Finally, for networks based on Weyl-Heisenberg filters, we determine the  prototype function bandwidth that minimizes---for fixed network depth $N$---the average number of operationally significant nodes per layer.
\end{abstract}



\section{INTRODUCTION}
\label{sec:intro}  
Feature extraction based on deep convolutional neural networks (CNNs) has been applied with significant success in a wide range of practical machine learning tasks [\citenum{Goodfellow-et-al-2016}]. Many of these applications, such as, e.g.,  the classification of images in the ImageNet data set, employ very deep networks with potentially hundreds of layers and $10$,$000$s of nodes [\citenum{he2015deep,simonyan2014very}] (e.g., the CNN in [\citenum{he2015deep}] has a depth of $152$ with an average number of $472$ nodes per layer). Such network sizes entail formidable computational challenges, both in the training phase due to the large number of parameters to be learned (e.g., the CNN in [\citenum{simonyan2014very}] has $144$ million parameters), and in operating the network due  to the large number of convolutions that need to be carried out (e.g., the CNN in [\citenum{he2015deep}] entails $11.3$ billion FLOPS to pass a single image through the network). Moreover,  storing the learned network parameters requires large amounts of memory. Very deep and wide CNNs may therefore not be suited to applications operating under strong resource constraints as is the case, e.g., in low-power embedded and mobile platforms [\citenum{Lane}]. It is hence important to understand the impact of CNN topology, specifically depth and width, on the network's feature extraction capabilities. 

We address this question for the class of scattering networks as introduced in [\citenum{MallatS}] and extended in [\citenum{Wiatowski_journal}]. Scattering network-based feature extractors were shown to yield classification performance competitive with the state-of-the-art on various data sets [\citenum{wiatowski2016discrete,CInC,ICASSP2016}]. Moreover, a  mathematical theory exists, which allows to establish formally that   such  feature extractors are---under certain technical conditions---horizontally [\citenum{MallatS}] or vertically [\citenum{Wiatowski_journal}] translation-invariant, energy-conserving [\citenum{Waldspurger,czaja2016uniform,WiatowskiEnergy}], deformation-stable in the sense of [\citenum{MallatS}] or exhibit limited sensitivity to deformations on input signal classes such as band-limited functions [\citenum{Wiatowski_journal}], cartoon functions [\citenum{grohs_wiatowski}], and Lipschitz functions [\citenum{grohs_wiatowski}].

Estimates of the number $N$ of layers (i.e., the network depth) needed to have $((1-\varepsilon)\cdot 100)\%$ of the input signal energy be contained in the feature vector---obtained by aggregating filtered versions of the propagated signals (a.k.a. feature maps)---were recently obtained in [\citenum{WiatowskiEnergy}]. The results in [\citenum{WiatowskiEnergy}] apply to  scattering networks  employing the modulus non-linearity, no pooling, and general filters that are analytic, constitute Parseval  frames [\citenum{Antoine}], and are allowed to be different in different network layers. The main findings of [\citenum{WiatowskiEnergy}] state that the feature map energy decays at least as fast as i) $\mathcal{O}(N^{-\alpha})$, for an explicitly specified $\alpha>0$, for general filters, ii) $\mathcal{O}((3/2)^{-N})$  for broad families of Weyl-Heisenberg (WH) filters, and iii) $\mathcal{O}((5/3)^{-N})$ for broad families of wavelet filters. 

\textit{Contributions}. For scattering networks that employ the modulus non-linearity and no pooling, we generalize the exponential energy decay results in [\citenum{WiatowskiEnergy}] to $\mathcal{O}(a^{-N})$,  where an \emph{arbitrary} decay factor $a>1$ can be realized by suitable choice of the WH prototype function or the mother wavelet. We then show how networks of fixed (possibly small) depth $N$ can be designed to guarantee that $((1-\varepsilon)\cdot 100)\%$ of the input signal's energy are contained in the feature vector. Based on the notion of operationally significant nodes, we characterize, partly rigorously and partly heuristically, the topology-reducing effects of (effectively) band-limited input signals, band-limited filters, and  feature map symmetries. The results we obtain suggest a classification into shallow, single-layer, constant-width, expanding-width, depth-pruned, and extremely narrow scattering networks. Finally, for networks based on WH filters, we determine the prototype function bandwidth that minimizes---for a fixed network depth $N$---the average number of operationally significant nodes per layer.

\section{CNN-BASED FEATURE EXTRACTORS}\label{architecture}
For the general notation employed in this paper, we refer to \cite[Section 1]{WiatowskiEnergy}. We set the stage by briefly reviewing the basics of scattering network-based feature extractors. The presentation follows closely that in \cite[Section 2]{WiatowskiEnergy}. Throughout the paper, we focus on the $1$-D case and employ the module sequence
\begin{equation*}\label{mods}
\Omega:=\big((\Psi,|\cdot|,\text{Id})\big)_{n\in \mathbb{N}},
\end{equation*}
i.e., each network layer is associated with (i) the same collection of filters  $\Psi = \{\chi \}\cup\{ g_{\lambda}\}_{\lambda \in \Lambda}\subseteq L^1(\R) \cap L^2(\R)$, where $\chi$,  referred to as  output-generating filter, and the $g_{\lambda}$, indexed by a countable  set $\Lambda$,  satisfy the Parseval frame condition [\citenum{Antoine}]
\begin{equation}\label{PFP}
\| f\ast\chi\|_2^2 +  \sum_{\lambda \in \Lambda}\| f\ast g_{\lambda}\|^2 = \| f\|_2^2,\quad \forall f\in L^2(\R),
\end{equation} (ii) the modulus non-linearity $|\cdot|:L^2(\R)\to L^2(\R)$, $|f|(x):=|f(x)|$, and (iii) no pooling, which, in the terminology of [\citenum{Wiatowski_journal}], corresponds to pooling through the identity operator with pooling factor equal to one. Associated with the module $(\Psi,|\cdot|,\text{Id})$, the operator $U[\lambda]$ defined in \cite[Eq. 12]{Wiatowski_journal} particularizes to
 \begin{equation}\label{eq:1}
U[\lambda]f=\big|f\ast g_{\la} \big|.
\end{equation}   
We extend \eqref{eq:1} to paths on index sets $$q=(\lambda_1,\lambda_2,\dots, \lambda_n) \in \underbrace{\Lambda\times \Lambda\times\dots\times \Lambda}_{n \text{ times }}=:\Lambda^n, \quad n \in \mathbb{N},$$ according to $U[q]f=\,U[(\lambda_1,\lambda_2,\dots, \lambda_n)]f\nonumber:=\, U[\lambda_n] \,\cdots \,U[\lambda_2]U[\lambda_1]f$, where, for the empty path $e:=\emptyset$, we set $\Lambda^0:=\{ e \}$ and $U[e]f:=f$, for  $f\in L^2(\R)$. The signals $U[q]f$ are often referred to  as  feature maps in the deep learning literature. The feature vector $\Phi_\Omega(f)$ is obtained by aggregating filtered versions of the feature maps. More formally, $\Phi_\Omega(f)$ is defined as \cite[Def. 3]{Wiatowski_journal}
\begin{equation}\label{ST}
\Phi_\Omega (f):=\bigcup_{n=0}^\infty\Phi^n_\Omega(f),
\vspace{-0.1cm}
\end{equation}
where $\Phi^n_\Omega(f):=\{ (U[q]f) \ast \chi \}_{q \in \Lambda^n}$ are the features generated in the $n$-th network layer, see Figure \ref{fig:gsn}. Here, $n = 0$ corresponds to the root of the network. The feature extractor\footnote{\small{Throughout, we refer to $\Phi_\Omega$ as \emph{feature extractor} and to $\Phi_\Omega(f)$ as \emph{feature vector} corresponding to the input signal $f$.}} $\Phi_\Omega$ was shown in \cite[Theorem 1]{Wiatowski_journal} to be vertically translation-invariant, provided though that pooling is  employed, with pooling factors $S_n\geq1$, $n\in \mathbb{N}$, (see \cite[Eq. 6]{Wiatowski_journal} for the definition of the general pooling operator) such that $\lim\limits_{N\to \infty} \prod_{n=1}^N S_n=\infty$. Moreover, $\Phi_\Omega$ exhibits limited sensitivity to certain non-linear deformations on input signal classes such as band-limited functions \cite[Theorem 2]{Wiatowski_journal}, cartoon functions \cite[Theorem 1]{grohs_wiatowski}, and Lipschitz functions \cite[Corollary 1]{grohs_wiatowski}. More recently, it was shown in \cite[Theorem 1]{WiatowskiEnergy} that $\Phi_\Omega$ is energy-conserving in the sense of the energy contained in the feature vector $\Phi_\Omega(f)$  being proportional to that of the corresponding input signal $f$.
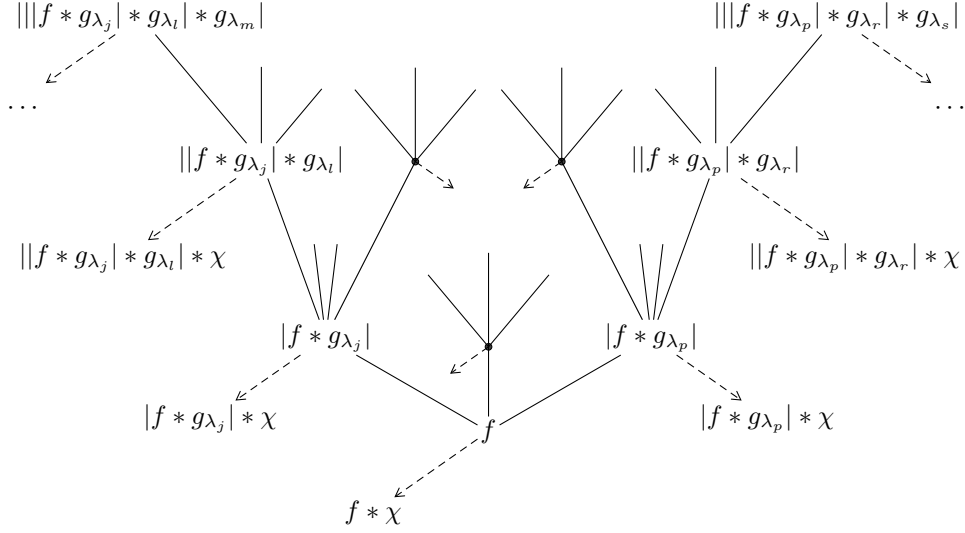
\begin{figure}[t!]
\centering
\begin{tikzpicture}[scale=2.5,level distance=10mm,>=angle 60]

  \tikzstyle{every node}=[rectangle, inner sep=1pt]
  \tikzstyle{level 1}=[sibling distance=30mm]
  \tikzstyle{level 2}=[sibling distance=10mm]
  \tikzstyle{level 3}=[sibling distance=4mm]
  \node {$f$}
	child[grow=90, level distance=.45cm] {[fill=gray!50!black] circle (0.5pt)
		child[grow=130,level distance=0.5cm] 
        		child[grow=90,level distance=0.5cm] 
        		child[grow=50,level distance=0.5cm]
		child[level distance=.25cm,grow=215, densely dashed, ->] {}  
	}
        child[grow=150] {node {$|f\ast g_{\lambda_j}|$}
	child[level distance=.75cm,grow=215, densely dashed, ->] {node {$|f\ast g_{\lambda_j}|\ast\chi$}
	}
	child[grow=83, level distance=0.5cm] 
	child[grow=97, level distance=0.5cm] 
        child[grow=110] {node {$||f\ast g_{\lambda_j}|\ast g_{\lambda_l}|$}
	child[level distance=.9cm,grow=215, densely dashed, ->] {node {$||f\ast g_{\lambda_j}|\ast g_{\lambda_l}|\ast\chi$}
	}
        child[grow=130] {node {$|||f\ast g_{\lambda_j}|\ast g_{\lambda_l}|\ast g_{\lambda_m}|$}
	child[level distance=0.75cm,grow=215, densely dashed, ->] {node[align=left]{\\ $\cdots$}}}
        child[grow=90,level distance=0.5cm] 
 	child[grow=50,level distance=0.5cm]
       }
       child[grow=63, level distance=1.05cm] {[fill=gray!50!black] circle (0.5pt)
	child[grow=130,level distance=0.5cm] 
       child[grow=90,level distance=0.5cm] 
       child[grow=50,level distance=0.5cm] 
       child[level distance=.25cm,grow=325, densely dashed, ->] {}     
       }
       }
       child[grow=30] {node {$|f\ast g_{\lambda_p}|$}
       child[level distance=0.75cm, grow=325, densely dashed, ->] {node {$|f\ast g_{\lambda_p}|\ast\chi$}
	}
	child[grow=83, level distance=0.5cm] 
	child[grow=97, level distance=0.5cm] 
        child[grow=117, level distance=1.05cm] {[fill=gray!50!black] circle (0.5pt)
        child[grow=130,level distance=0.5cm] 
        child[grow=90,level distance=0.5cm] 
        child[grow=50,level distance=0.5cm] 
        child[level distance=.25cm,grow=215, densely dashed, ->] {}  
	 }
        child[grow=70] {node {$||f\ast g_{\lambda_p}|\ast g_{\lambda_r}|$}
	 child[level distance=0.9cm,grow=325, densely dashed, ->] {node {$||f\ast g_{\lambda_p}|\ast g_{\lambda_r}|\ast\chi$}}
	child[grow=130,level distance=0.5cm] 
         child[grow=90,level distance=0.5cm] 
             child[grow=50] {node {$|||f\ast g_{\lambda_p}|\ast g_{\lambda_r}|\ast g_{\lambda_s}|$}
             child[level distance=0.75cm,grow=325, densely dashed, ->] {node[align=left]{\\ $\cdots$}}}
	}
     }
	child[level distance=0.75cm, grow=215, densely dashed, ->] {node {$f\ast \chi$}};
\end{tikzpicture}
\vspace{0.4cm}
\caption{\small{Network architecture underlying the feature extractor  \eqref{ST}.}} 
\label{fig:gsn}
\end{figure}

\section{FEATURE MAP ENERGY DECAY}\label{Sec:prop}
The total energy contained in the feature maps in the $n$-th network layer is given by 
 \begin{equation*}\label{ene}
 W_n(f):=\sum_{q\in \Lambda^n}\| U[q]f\|_2^2, \quad f\in L^2(\R).
 \end{equation*}
 Our goal is to construct WH and wavelet filters that realize exponential energy decay according to $W_n(f)=\mathcal{O}(a^{-n})$, with arbitrary $a>1$. In particular, we want to tune the decay factor $a$ by adjusting a single parameter, which will be seen to determine the WH prototype function or the mother wavelet  bandwidth. This will be accomplished through the following constructions employed throughout the paper:
 \begin{itemize}
\item[i)]{\emph{WH filters}: For fixed $R>0$, $\delta \geq \frac{R}{2}$, let the functions $g,\phi \in L^1(\R)\cap L^2(\R)$ satisfy the Littlewood-Paley condition
\begin{equation*}\label{eq:Gaborcondi}
 |\widehat{\phi}(\omega)|^2+\sum_{k=1}^{\infty} |\widehat{g}(\omega-(Rk+\delta))|^2=1,\hspace{0.5cm} a.e. \ \omega\geq 0,
 \end{equation*}
 with $\text{supp}(\widehat{g})= [-R,R]$, $\widehat{g}(-\omega)=\widehat{g}(\omega)$, and $\widehat{g}$ real-valued.  
Moreover, let $g_k(x):=e^{2\pi i (Rk+\delta)x}g(x)$, $k\geq1$,  $g_k(x):=e^{-2\pi i (R|k|+\delta)x}g(x)$, $k\leq-1$, and set $\chi(x):=\phi(x)$, $x \in \R$. The Fourier transforms $\widehat{g_k}$ and $\widehat{g}$ are illustrated in Figure \ref{fig123}. }
\item[ii)]{\emph{Wavelets}: For fixed $r>1$, let the mother and father wavelets $\psi,\phi \in L^1(\R)\cap L^2(\R)$ satisfy the Littlewood-Paley condition [\citenum{Paley}]
\begin{equation*}\label{eq:Waveletcondi}
 |\widehat{\phi}(\omega)|^2+\sum_{j=1}^{\infty} |\widehat{\psi}(r^{-j}\omega)|^2=1,\hspace{0.5cm} a.e. \ \omega\geq 0,
 \end{equation*}
with  $\text{supp}(\widehat{\psi})= [r^{-1},r]$ and $\widehat{\psi}$ real-valued. Moreover, let $g_j(x):=r^{j}\psi(r^jx)$, $j\geq1$,  $g_j(x):=r^{|j|}\psi(-r^{|j|}x)$, $j\leq-1$, and let the output-generating filter be  $\chi(x):=\phi(x)$, $x \in \R$. The Fourier transforms of the wavelets $g_j$ and  the mother wavelet $\psi$ are illustrated in Figure \ref{fig8}.}
\end{itemize}

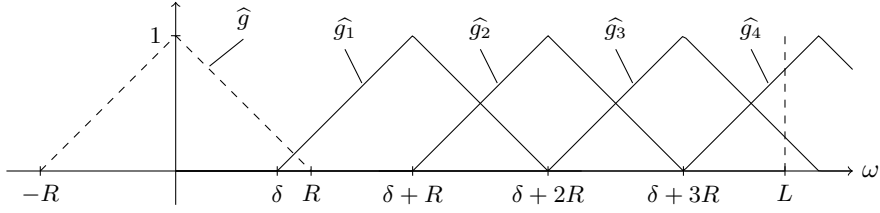
\begin{figure}
\begin{center}

\begin{tikzpicture}
	\begin{scope}[scale=.9]
	\draw[->] (-2.5,0) -- (10,0) node[right] {$\omega$};
	\draw[->] (0,-.5) -- (0,2.5) node[above] {};
%
		\draw (1.5 cm,2pt) -- (1.5 cm,-2pt) node[anchor=north] {\footnotesize$\delta$};
		\draw (-2 cm,2pt) -- (-2 cm,-2pt) node[anchor=north] {\footnotesize$-R$};
		\draw (2 cm,2pt) -- (2 cm,-2pt) node[anchor=north] {\footnotesize$R$};

		\draw (3.5 cm,2pt) -- (3.5 cm,-2pt) node[anchor=north] {\footnotesize$\delta+R$};
		\draw (5.5 cm,2pt) -- (5.5 cm,-2pt) node[anchor=north] {\footnotesize$\delta+2R$};
		\draw (7.5 cm,2pt) -- (7.5 cm,-2pt) node[anchor=north] {\footnotesize$\delta+3R$};
		\draw (9 cm,2pt) -- (9 cm,-2pt) node[anchor=north] {\footnotesize$L$};

		\draw (2pt, 2 cm) -- (-2pt,2 cm) node[anchor=east] {\footnotesize$1$};
			
			\draw (2.75 cm,1.3cm) -- (2.5 cm,1.8cm) node[anchor=south] {\small$\widehat{g_1}$};
			\draw (4.75 cm,1.3cm) -- (4.5 cm,1.8cm) node[anchor=south] {\small$\widehat{g_2}$};
			\draw (6.75 cm,1.3cm) -- (6.5 cm,1.8cm) node[anchor=south] {\small$\widehat{g_3}$};
			\draw (8.75 cm,1.3cm) -- (8.5 cm,1.8cm) node[anchor=south] {\small$\widehat{g_4}$};
			\draw (.5 cm,1.55cm) -- (1cm,1.95cm) node[anchor=south] {\small$\widehat{g}$};

        		\draw[dashed] (9,0) -- (9,2);
                 \draw[scale=1,domain=-2:2,dashed,variable=\x,samples=200] plot ({\x},{2*max(0,1-abs(\x/2))});
                  \draw[scale=1,domain=0:6,variable=\x,samples=200] plot ({\x},{2*max(0,1-abs((\x-3.5)/2))});
                  \draw[scale=1,domain=3.:9,variable=\x,samples=200] plot ({\x},{2*max(0,1-abs((\x-5.5)/2))});
                  \draw[scale=1,domain=0:10,variable=\x,samples=200] plot ({\x},{2*max(0,1-abs((\x-7.5)/2))});
                  \draw[scale=1,domain=0:10,variable=\x,samples=200] plot ({\x},{2*max(0,1-abs((\x-9.5)/2))});

	\end{scope}
\end{tikzpicture}

\end{center}
\caption{\small{Illustration of the Fourier transforms of the WH filters $g_k$ on the frequency band $[0,L]$. The Fourier transform $\widehat{g}$  of the prototype function  $g$  is supported on the interval $[-R,R]$. }}
\label{fig123}
\end{figure}
The conditions we impose can be satisfied by construc\-ting $g$, $\phi$ in i) from a function whose Fourier transform is a $1$-D $B$-spline \cite[Section 1]{Gaborexmpl}, and  $\psi,\phi$ in ii)  from, e.g., the analytic Meyer wavelet \cite[Section 3.3.5]{Daubechies}. We emphasize that both the WH and the wavelet filters satisfy---by construction---the analyticity and highpass condition \cite[Assumption 1]{WiatowskiEnergy} as well as the symmetry property
\begin{equation}\label{sdlsdhkdjfshf}
\widehat{g_{\lambda}}(-\omega)=\widehat{g_{-\lambda}}(\omega), \quad \forall \,\lambda\in \mathbb{Z}\backslash\{ 0\}, \ \forall \, \omega\in \R,
\end{equation}
 which will turn out (in Section \ref{lkfnkjfjkewh}) to be key in reducing the number of ``operationally relevant nodes'', a notion defined in \eqref{sdfdddddas4444} below.  We refer to  the intervals $[-\delta,\delta]$ and $[-1,1]$ as ``spectral gaps'' left by the WH and wavelet filters,  respectively,  as we have $\supp{(\widehat{g_k}\,)}\cap[-\delta,\delta]=\emptyset$, for all $k\in \Z\backslash\{0\}$, in the WH case, and $\supp{(\widehat{g_j}\,)}\cap[-1,1]=\emptyset$, for all $j\in \Z\backslash\{0\}$, in the wavelet case.
 
The results in this section and in Section \ref{fslksflksdfjklj} apply to input signals that belong to the class of Sobolev functions $H^s(\R)=\big\{f\in L^2(\R) \ \big| \  \int_{\R}(1+|\omega|^2)^{s}|\widehat{f}(\omega)|^2\mathrm d\omega<\infty \big\}$, $s\geq 0$, where the parameter $s$ acts as a smoothness index.  Sobolev functions encompass a wide range of practically relevant signal classes such as square-integrable functions $L^2(\R)=H^{0}(\R)$,  strictly $(L)$-band-limited functions $L_L^2(\R)\subseteq H^s(\R)$, for all $L>0$ and all $s\geq 0$, and cartoon functions  [\citenum{Cartoon}] $\mathcal{C}^{K}_{\mathrm{CART}}\subseteq H^s(\R)$, for all $K>0$ and all $s\in (0,\frac{1}{2})$ (see \cite[Lemma 1]{WiatowskiEnergy}). We note that cartoon functions  are widely used in the mathematical signal processing literature [\citenum{grohs_wiatowski,wiatowski2016discrete,WiatowskiEnergy,Grohs_alpha}] as a model for natural images such as, e.g., images of handwritten digits [\citenum{MNIST}].

Our first main result is the following.
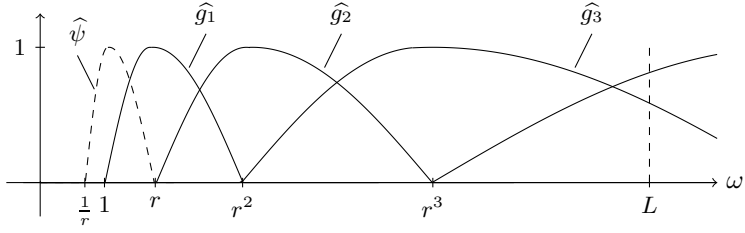
\begin{figure}
\begin{center}
\begin{tikzpicture}
	\begin{scope}[scale=.9]
	\draw[->] (-.25,0) -- (10,0) node[right] {$\omega$};
	\draw[->] (0,-.5) -- (0,2.5) node[above] {};
		     		\draw[dashed] (9,0) -- (9,2);
		\draw(-0.5,0) -- (10,0);
				\draw (9 cm,2pt) -- (9 cm,-2pt) node[anchor=north] {\footnotesize$L$};
		\draw (.66 cm,2pt) -- (.66 cm,-2pt) node[anchor=north] {\small$\frac{1}{r}$};
		\draw (.95 cm,2pt) -- (.95 cm,-2pt) node[anchor=north] {\small$1$};
		\draw (1.7 cm,2pt) -- (1.7 cm,-2pt) node[anchor=north] {\small$r$};
		\draw (2.99 cm,2pt) -- (2.99 cm,-2pt) node[anchor=north] {\small$r^2$};
		\draw (5.8 cm,2pt) -- (5.8 cm,-2pt) node[anchor=north] {\small$r^3$};
		\draw (2pt, 2 cm) -- (-2pt,2 cm) node[anchor=east] {\small$1$};

		\draw (2.2 cm,1.7cm) -- (2.45 cm,2.2cm) node[anchor=south] {\small$\widehat{g_1}$};
		\draw (4.2 cm,1.7cm) -- (4.45 cm,2.2cm) node[anchor=south] {\small$\widehat{g_2}$};
		\draw (7.9 cm,1.7cm) -- (8.15 cm,2.2cm) node[anchor=south] {\small$\widehat{g_3}$};
		\draw (.8 cm,1.38cm) -- (.55cm,1.88cm) node[anchor=south] {\small$\widehat{\psi}$};

                
                \draw[scale=1,domain=0:1,dashed,variable=\x,samples=100] plot ({\x},{max(2*sin(115*(1.57*(3*abs(\x/2)-1))),0)});
                \draw[scale=1,domain=1:2,dashed,variable=\x,samples=150] plot ({\x},{max(0,2*cos(115*(1.57*(1.5*abs((\x+0.0125)/2+0.15)-1))))});
                
                \draw[scale=1,domain=0:1.64,variable=\x,samples=100] plot ({\x},{max(2*sin(115*(1.57*(3*abs((\x+0.375)/4)-1))),0)});
                \draw[scale=1,domain=1.6395:5,variable=\x,samples=150] plot ({\x},{max(0,2*cos(115*(1.57*(1.5*abs((\x+0.025+0.375)/4+0.15)-1))))});
                
                \draw[scale=1,domain=0:3.1,variable=\x,samples=100] plot ({\x},{max(2*sin(115*(1.57*(3*abs((\x+2*0.475)/8)-1))),0)});
                \draw[scale=1,domain=3.09:7,variable=\x,samples=150] plot ({\x},{max(0,2*cos(115*(1.57*(1.5*abs((\x+0.05+2*0.475)/8+0.15)-1))))});
                
                \draw[scale=1,domain=0:5.5,variable=\x,samples=100] plot ({\x},{max(2*sin(115*(1.57*(3*abs((\x+4*0.59)/16)-1))),0)});
                \draw[scale=1,domain=5.5:10,variable=\x,samples=150] plot ({\x},{max(0,2*cos(115*(1.57*(1.5*abs((\x+0.1+4*0.59)/16+0.15)-1))))});
                
      		\draw[scale=1,domain=5.8:10,variable=\x,samples=100] plot ({\x},{max(2*sin(115*(1.57*(3*abs((\x+8*0.61)/32)-1))),0)});


	\end{scope}
\end{tikzpicture}

\end{center}
\caption{\small{Illustration of the Fourier transforms of the wavelet filters $g_j$ on the frequency band $[0,L]$. The Fourier transform $\widehat{\psi}$  of the mother wavelet $\psi$  is supported on the interval $[r^{-1},r]$. }}
\label{fig8}
\end{figure}
\begin{theorem}\label{thm4}
For the WH case, let $R>0$, $\delta\geq \frac{R}{2}$, and set 
\begin{equation}\label{432897498237}
a=\frac{1}{2} + \frac{\delta}{R}.
\end{equation}
For the wavelet case, let $r>1$ and set \begin{equation}\label{4328974982371}
a=\frac{r^2+1}{r^2-1}.
\end{equation}
 Then, in both cases, for every Sobolev function $f\in H^s(\R)$, $s>0$, we have 
\begin{equation}\label{main_sow11}
W_n(f)=\mathcal{O}\big({a}^{-\gamma n}\big),
\end{equation}
where $\gamma:=\min\{ 1,2s\}$.
\end{theorem}
\begin{proof}
The proof  is structurally very similar to that of \cite[Theorem 2]{WiatowskiEnergy} and will hence not be presented in detail. In a nutshell, the new elements needed to establish \eqref{main_sow11} are, for the WH case, to replace the so-called modulation weights $\nu_k:=Rk+\frac{2}{3}R$, $k\geq 1$,  $\nu_k:=-\nu_{|k|}$, $k\leq-1$, defined in \cite[Eq. 139]{WiatowskiEnergy} by $\nu_k:=Rk+\delta-\frac{R^2}{R+2\delta}$, $ k\geq 1$, $\nu_k:=-\nu_{|k|}$, $  k\leq-1,$ and similarly, in the wavelet case, to replace the modulation weights $\nu_j:=\frac{4}{5}2^{j}$, $j\geq 1$,  $\nu_j:=-\frac{4}{5}2^{|j|}$, $j\leq-1$, defined in \cite[Eq. 115]{WiatowskiEnergy} by the $r$-dependent modulation weights $\nu_j:=\frac{2r}{r^2+1}r^{j}$, $ j\geq 1$,   $\nu_j:=-\frac{2r}{r^2+1}r^{|j|}$, $j\leq-1$.
The rest of the proof follows closely that of \cite[Theorem 2]{WiatowskiEnergy}. \end{proof}
The identities \eqref{432897498237} and \eqref{4328974982371} show that the filter constructions we propose, indeed, allow to tune the decay factor $a$ through a single parameter, namely $R$ in the WH case  and $r$ in the wavelet case. Reducing $R$,$\,r$ results in faster energy decay (see also Figure \ref{fig6}). Particularizing \eqref{432897498237} for $R=\delta$ and \eqref{4328974982371} for $r=2$  recovers the decay factors $a=3/2$ and $a=5/3$, respectively, established in  \cite[Theorem 2]{WiatowskiEnergy}. Finally, we refer the reader to \cite[Section 3]{WiatowskiEnergy} and references therein for an overview of  previous work on the decay rate of $W_n(f)$.

\section{DEPTH-CONSTRAINED SCATTERING NETWORKS}\label{fslksflksdfjklj}
We now turn to the design of  scattering networks of fixed (possibly small) depth $N$ that capture most of the input signal's features. This will be formalized by seeking WH and wavelet filters that, for given $\varepsilon>0$ and given depth $N\in \mathbb{N}$, result in feature extractors satisfying\footnote{\small{The feature space norm is defined as $|||\Phi^n_\Omega(f)|||^2:= \sum_{q \in  \Lambda^n} \hspace{-0.05cm}\| (U[q]f)\ast \chi\|_2^2$. }}
\begin{equation}\label{bubu_main}
(1-\varepsilon)\|f\|_2^2\leq \sum_{n=0}^{N} |||\Phi_\Omega^n(f)|||^2\leq \|f\|_2^2,\quad \forall f\in L^2(\R).
\end{equation}
The lower bound in \eqref{bubu_main} guarantees that at least  $((1-\varepsilon)\cdot 100)\%$ of the input signal energy are contained in the feature vector $\{ \Phi_\Omega^n(f)\}_{n=0}^{N}$ generated in the first $N$ network layers. We note that establishing the upper bound in \eqref{bubu_main} does not pose any significant difficulties as it follows straight from the results in \cite[Appendix E]{Wiatowski_journal}. The lower bound in \eqref{bubu_main} implies a trivial null-set for the  feature extractor $\Phi_\Omega$  and thereby ensures that the only signal $f$ that is mapped to the all-zeros feature vector is $f=0$. We emphasize that the energy decay results in Theorem \ref{thm4} pertain to the feature maps $U[q]f$, whereas energy conservation according to \eqref{bubu_main} applies to the feature vector $\{ \Phi_\Omega^n(f)\}_{n=0}^{N}$.
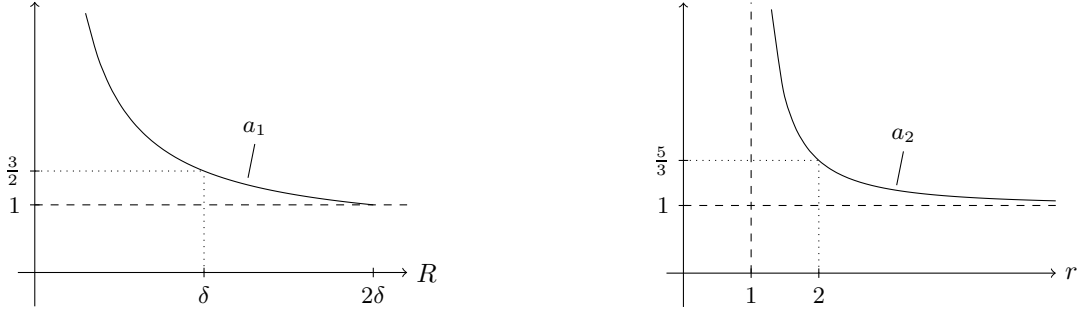
\begin{figure}
\begin{center}
\begin{minipage}{.5\textwidth}
\centering

\begin{tikzpicture}
	\begin{scope}[scale=.9]
	\draw[->] (-.25,0) -- (5.5,0) node[right] {$R$};
	\draw[->] (0,-.5) -- (0,4) node[above] {};
	\draw[dashed] (0,1cm) -- (5.5,1cm);
	\draw[dotted] (0,1.5cm) -- (2.5,1.5cm);
	\draw[dotted] (2.5cm,0) -- (2.5cm,1.5cm);
	
		\draw (2.5 cm,2pt) -- (2.5 cm,-2pt) node[anchor=north] {\small$\delta$};
		\draw (5 cm,2pt) -- (5 cm,-2pt) node[anchor=north] {\small$2\delta$};
		\draw (2pt, 1 cm) -- (-2pt,1 cm) node[anchor=east] {\small$1$};
		\draw (2pt, 1.5 cm) -- (-2pt,1.5 cm) node[anchor=east] {\small$\frac{3}{2}$};
			
			\draw (3.15 cm,1.4cm) -- (3.25 cm,1.9cm) node[anchor=south] {\small$a_1$};			

        \draw[scale=1,domain=.75:5,smooth,variable=\x] plot ({\x},{.5 + 2.5/\x});
	\end{scope}
\end{tikzpicture}

\end{minipage}%
\begin{minipage}{.5\textwidth}
\centering
\begin{tikzpicture}
	\begin{scope}[scale=.9]
	\draw[->] (-.25,0) -- (5.5,0) node[right] {$r$};
	\draw[->] (0,-.5) -- (0,4) node[above] {};
	\draw[dashed] (0,1cm) -- (5.5,1cm);
	\draw[dashed] (1cm,0) -- (1cm,4cm);
	\draw[dotted] (0,1.666cm) -- (2,1.666cm);
	\draw[dotted] (2cm,0) -- (2cm,1.666cm);
	
		\draw (1 cm,2pt) -- (1 cm,-2pt) node[anchor=north] {\small$1$};
		\draw (2 cm,2pt) -- (2 cm,-2pt) node[anchor=north] {\small$2$};
		\draw (2pt, 1 cm) -- (-2pt,1 cm) node[anchor=east] {\small$1$};
		\draw (2pt, 1.666 cm) -- (-2pt,1.666 cm) node[anchor=east] {\small$\frac{5}{3}$};
			
			\draw (3.15 cm,1.3cm) -- (3.25 cm,1.8cm) node[anchor=south] {\small$a_2$};			

        \draw[scale=1,domain=1.3:5.5,smooth,variable=\x] plot ({\x},{(\x*\x+1)/(\x*\x-1)});
	\end{scope}

\end{tikzpicture}
\end{minipage}

\end{center}
\caption{\small{Illustration of the functions $a_1(R):=\frac{1}{2}+\frac{\delta}{R}$, for  $R\leq 2\delta$, (left plot) and $a_2(r):=\frac{r^2+1}{r^2-1}$, for $r>1$, (right plot). }}
\label{fig6}
\end{figure}

The next result explains how to choose $R$ in the WH and $r$ in the wavelet case so as to satisfy \eqref{bubu_main}. In particular, we shall see that for every (possibly small) $\varepsilon>0$ and every $N\in \mathbb{N}$, say $\varepsilon=0.01$ and $N=1$, there exist $R>0$ and $r>1$ such that  \eqref{bubu_main} holds. 
\begin{theorem}\label{thm6}
For the WH case, let $R>0$, $\delta \geq \frac{R}{2}$. For the wavelet case,  let $r>1$ and $\delta=1$. Moreover, take $f\in H^s(\R)\backslash\{ 0\}$, $s>0$, fix $\varepsilon \in (0,1)$ and $N\in \N$, let $l > \frac{1}{2}\, \varepsilon^{1/\gamma}\,\delta$, where $\gamma:= \min\{ 1,2s\}$, and define 
 \begin{equation*}\label{sdflkskkh}
\kappa:=\Bigg(\frac{2l\, \| f\|_{H^s}^{2/\gamma}}{\varepsilon^{1/\gamma}\delta\|f\|_2^{2/\gamma}}\Bigg)^{1/N}.
\end{equation*}
Then,  \eqref{bubu_main} holds in the WH case, for 
\begin{equation}\label{hihih2qq111}
0<R\leq \frac{\delta}{\kappa-\frac{1}{2}},
\end{equation}
and, in the wavelet case, for
\begin{equation}\label{hihih2qq}
1<r\leq \sqrt{\frac{\kappa+1}{\kappa-1}}.
\end{equation}
\end{theorem}

\begin{proof} Let $a$ be the decay factor in \eqref{432897498237} or \eqref{4328974982371}. Then, it follows from \cite[Corollary 2]{WiatowskiEnergy} that 
\begin{equation}\label{slkjlkwj}
a \geq \Bigg(\frac{2l\, \| f\|_{H^s}^{2/\gamma}}{\varepsilon^{1/\gamma}\delta\|f\|_2^{2/\gamma}}\Bigg)^{1/N}=\kappa
\end{equation}
is sufficient for \eqref{bubu_main} to hold. In the WH case, we have $a = \frac{1}{2} +\frac{\delta}{R}$, $\delta\geq \frac{R}{2}$, which,  when combined with \eqref{slkjlkwj}, yields
\begin{equation}\label{slkjlkwj222}
\frac{1}{2} +\frac{\delta}{R} \geq \kappa.
\end{equation}
Rearranging terms in \eqref{slkjlkwj222} establishes \eqref{hihih2qq111}. Next, in the wavelet case, we have $a = \frac{r^2+1}{r^2-1}$, $r>1$, 
which, when combined with \eqref{slkjlkwj}, leads to
\begin{equation}\label{slkjlkwj2}
\frac{r^2+1}{r^2-1} \geq \kappa.
\end{equation}
Finally, rearranging terms in \eqref{slkjlkwj2} establishes \eqref{hihih2qq}  and thereby completes the proof.
\end{proof}

\section{NUMBER OF OPERATIONALLY SIGNIFICANT NODES}\label{lkfnkjfjkewh}


While the results presented thus far were of mathematically strict nature, in the present section, we shall allow ourselves to argue on a less formal level. The energy decay and conservation results established so far assume an infinite number of filters in the module $(\Psi,|\cdot|,\text{Id})$, and hence an infinite number of nodes in each network layer. Formally, this is a consequence of the filters $\{g_{\lambda}\}_{\lambda\in \Lambda}$ depending on an index set $\Lambda$ with  $\text{card}(\Lambda)=\infty$, which, in turn, is needed to satisfy the frame condition \eqref{PFP}. However, real-world input signals $f$ can be considered effectively band-limited \cite{Slepian}, i.e., the support region that contains most of the energy of $\widehat{f}$---denoted by $\text{esupp}(\widehat{f})$---satisfies $\text{esupp}(\widehat{f})=[-L,L]$, for some $L>0$; we shall refer to $L$ as effective bandwidth. Since the function $f\ast g_\lambda$ is strictly band-limited (owing to $g_\lambda$ strictly band-limited, by assumption), and the modulus non-linearity results in (roughly) a doubling of bandwidth, as heuristically argued below, we allow ourselves to assume in the remainder of the paper that all feature maps are effectively band-limited. Consequently, it is sensible to ask how many nodes are actually needed in the $n$-th network layer to capture the feature map energy contained in $\text{esupp}(\widehat{U[q]f})$, $q\in \Lambda^{n-1}$. We formalize this question by defining the number of \emph{operationally significant nodes} in the $n$-th network layer as  
\begin{align}
\Xi(n):=&\,\text{card}\Big(\Big\{ U[q]f \, \Big| \, q \in \Lambda^{n}_{\text{sig}}\Big\}\Big),\quad n\geq 0,\label{sdfdddddas4444}
\end{align}
where the set $\Lambda^{n}_{\text{sig}}$ is defined (recursively) according to $\Lambda^{0}_{\text{sig}}:=\Lambda^0$, $\Lambda^{1}_{\text{sig}}:=\big\{ \lambda \in \Lambda \ \big| \ \text{esupp}(\widehat{f}\,)\cap\text{supp}(\widehat{g_\lambda}) \neq \emptyset\big\},$ and 
\begin{equation}\label{selkflklksdkdsnl}
\Lambda^{n}_{\text{sig}}:= \bigg\{ (q,\lambda) \ \bigg|\ q\in \Lambda^{n-1}_{\text{sig}} \ \text{ and } \ \lambda \in \Lambda \ \text{  with  } \ \text{esupp}(\widehat{U[q]f})\cap\text{supp}(\widehat{g_\lambda}) \neq \emptyset\bigg\},\quad n\geq2.
\end{equation}
For the root of the network, i.e., $n=0$, we have $\Xi(0)=1$, owing to $U[q]f=U[e]f=f$.  The definition of $\Xi(n)$ accounts for  a topology reduction, relative to the full tree in Figure \ref{fig:gsn}, caused by i) feature map symmetries  (see \eqref{dfndkslflkdsnfff} below) and reflected by counting the number of distinct\footnote{We recall that the cardinality of a set equals the number of \emph{distinct} elements in the set, e.g., $\text{card}(\{ a,a,b\})=2$, for $a\neq b$.} feature maps $U[q]f$ in \eqref{sdfdddddas4444}  only\footnote{We emphasize that the location (in the full tree in Figure \ref{fig:gsn}) of  identical feature maps (not counted in $\Xi(n)$) is uniquely determined. In practice, it therefore suffices to, indeed, compute these features only once and arrange identical copies accordingly in $\Phi_\Omega(f)$.}, and ii) width pruning owing to strict band-limitation of the filters $g_\lambda$ and hence effective band-limitation of the feature maps $U[q]f$. Note that the specific spectral structure of $f$, e.g., a multi-band structure, can lead to further topology reduction. This effect will, however, not be taken into account in the remainder of the paper.  We honor the dependence of $\Xi(n)$  on i) the filters in $\Psi$ (and their parameters $\delta$, $r$, and $R$) and ii) the effective bandwidth $L$ of the input signal $f$ by employing the notation $\Xi_\text{WH}(n,R,\delta,L)$ and $\Xi_{\text{wav}}(n,r,L)$ wherever appropriate.

\begin{figure}
\begin{center}

\begin{tikzpicture}
	\begin{scope}[scale=.9]
	\draw[->] (-7.5,0) -- (7.5,0) node[right] {$\omega$};
	\draw[->] (0,-.5) -- (0,2.5) node[above] {};
%
		\draw (1.5 cm,2pt) -- (1.5 cm,-2pt) node[anchor=north] {\footnotesize$\delta$};
		\draw (-1.5 cm,2pt) -- (-1.5 cm,-2pt) node[anchor=north] {\footnotesize$-\delta$};

		\draw (3.5 cm,2pt) -- (3.5 cm,-2pt) node[anchor=north] {\footnotesize$\delta+R$};
		\draw (5.5 cm,2pt) -- (5.5 cm,-2pt) node[anchor=north] {\footnotesize$\delta+2R$};
		\draw (-3.5 cm,2pt) -- (-3.5 cm,-2pt) node[anchor=north] {\footnotesize$-(\delta+R)$};
		\draw (-5.5 cm,2pt) -- (-5.5 cm,-2pt) node[anchor=north] {\footnotesize$-(\delta+2R)$};
		\draw (.5 cm,2pt) -- (.5 cm,-2pt) node[anchor=north] {\footnotesize$L$};
		\draw (-.5 cm,2pt) -- (-.5 cm,-2pt) node[anchor=north] {\footnotesize$-L$};

		\draw (2pt, 2 cm) -- (-2pt,2 cm) node[anchor=south east] {\footnotesize$1$};
			
			\draw (2.75 cm,1.3cm) -- (2.5 cm,1.8cm) node[anchor=south] {\small$\widehat{g_{_1}}$};
			\draw (-2.75 cm,1.3cm) -- (-2.5 cm,1.8cm) node[anchor=south] {\small$\widehat{g_{_{-1}}}$};

        		\draw[dashed] (.5,0) -- (.5,2);
		\draw[dashed] (-.5,0) -- (-.5,2);
		\draw[dashed] (-.5,2) -- (.5,2);
                  \draw[scale=1,domain=1.5:5.5,variable=\x,samples=200] plot ({\x},{2*max(0,1-abs((\x-3.5)/2))});
                  \draw[scale=1,domain=-5.5:-1.5,variable=\x,samples=200] plot ({\x},{2*max(0,1-abs((\x+3.5)/2))});

	\end{scope}
\end{tikzpicture}

\begin{tikzpicture}
	\begin{scope}[scale=.9]
	\draw[->] (-7.5,0) -- (7.5,0) node[right] {$\omega$};
	\draw[->] (0,-.5) -- (0,2.5) node[above] {};
%
		\draw (.5 cm,2pt) -- (.5 cm,-2pt) node[anchor=north] {\footnotesize$\delta$};
		\draw (-.5 cm,2pt) -- (-.5 cm,-2pt) node[anchor=north] {\footnotesize$-\delta$};

		\draw (3.5 cm,2pt) -- (3.5 cm,-2pt) node[anchor=north] {\footnotesize$  \ \ \ \delta+(k-1)R$};
		\draw (-3.5 cm,2pt) -- (-3.5 cm,-2pt) node[anchor=north] {\footnotesize$-(\delta+(k-1)R) \ \ \ \ \ \  \ \ \ $};
		\draw (2. cm,2pt) -- (2. cm,-2pt) node[anchor=north] {\footnotesize$L$};
		\draw (-2. cm,2pt) -- (-2. cm,-2pt) node[anchor=north] {\footnotesize$-L$};

		\draw (2pt, 2 cm) -- (-2pt,2 cm) node[anchor=south east] {\footnotesize$1$};
			
			\draw (4.75 cm,1.3cm) -- (4.5 cm,1.8cm) node[anchor=south] {\small$\widehat{g_{_{k}}}$};
			\draw (-4.75 cm,1.3cm) -- (-4.5 cm,1.8cm) node[anchor=south] {\small$\widehat{g_{_{-k}}}$};

        		\draw[dashed] (2.,0) -- (2.,2);
		\draw[dashed] (-2.,0) -- (-2.,2);
		\draw[dashed] (-2.,2) -- (2.,2);
                  \draw[scale=1,domain=3.5:7.,variable=\x,samples=200] plot ({\x},{2*max(0,1-abs((\x-5.5)/2))});
                  \draw[scale=1,domain=-7.:-3.5,variable=\x,samples=200] plot ({\x},{2*max(0,1-abs((\x+5.5)/2))});

	\end{scope}
\end{tikzpicture}

\end{center}
\caption{\small{Top row: If $L\leq\delta$, then $\text{esupp}(\widehat{f}\,)=[-L,L]$ is fully contained in the spectral gap $[-\delta,\delta]$ left by the filters $\{ g_k\}_{k\in \mathbb{Z}\backslash\{ 0\}}$. Bottom row: If $L>\delta$ and $|k|> \lfloor (L-\delta)R^{-1} + 1 \rfloor $, the spectral supports of the filters $g_k$ do not overlap with $\text{esupp}(\widehat{f}\,)=[-L,L]$.}}
\label{fig1234}
\end{figure}
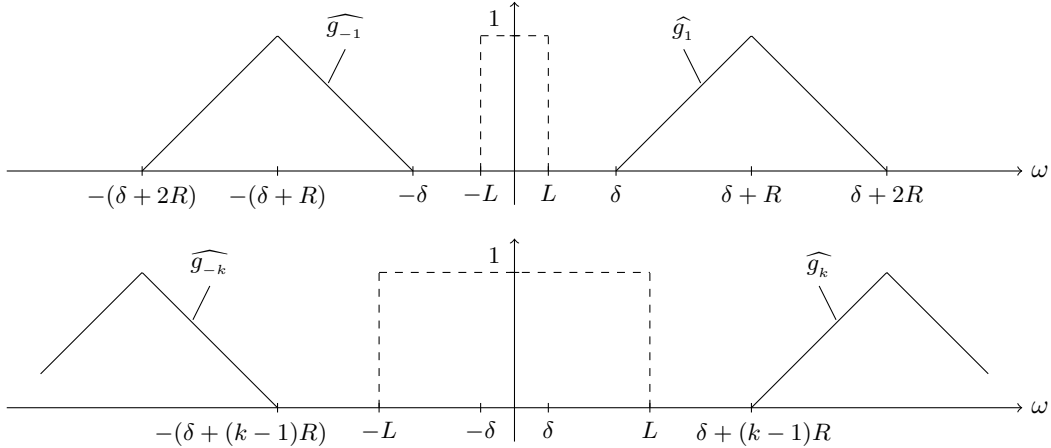

Next, our goal is to determine $\Xi(n)$, for $n\geq1$. Starting with the WH case and $n=1$, we have
$\text{esupp}(\widehat{f}\,)\, \cap\,\text{supp}(\widehat{g_k})\neq \emptyset
$ if $L>\delta$ (which prevents $\text{esupp}(\widehat{f}\,)=[-L,L]$ from being fully contained in the spectral gap $[-\delta,\delta]$ left by the filters $\{ g_k\}_{k\in \mathbb{Z}\backslash\{ 0\}}$, see Figure \ref{fig1234}, top row) and   
$|k|\leq \lfloor (L-\delta)R^{-1} + 1\rfloor$ (see Figure \ref{fig1234}, bottom row). This yields 
\begin{equation*}\label{sdfnsdkjfhaaaa}
\Lambda^{1}_{\text{WH}, \, \text{sig}}=\big\{k \in \mathbb{Z}\backslash \{ 0\} \ \big| \ |k|\leq \lfloor (L-\delta)R^{-1}+1\rfloor \big\},\hspace{1cm} \text{if } L>\delta,
\end{equation*}
and $\Lambda^{1}_{\text{WH}}=\emptyset$, if $L\leq \delta$, which, in turn, implies 
\begin{equation*}\label{sdfnsdkjfhaaaa}
\Xi_{\text{WH}}(1,R,\delta,L)
=2 \lfloor (L-\delta)R^{-1}+1\rfloor, \hspace{1cm}\text{if } L>\delta,
\end{equation*}
and $\Xi_{\text{WH}}(1,R,\delta,L)=0$, if $L\leq \delta$. Next, determining $\Xi(2)$ requires, by \eqref{selkflklksdkdsnl}, studying the spectral characteristics of the feature maps $U[k]f=|f\ast g_k|$. We note that, owing to the modulus non-linearity, characterizing the effective spectral support of $U[k]f=|f\ast g_k|$ is non-trivial. We can, however, take a cue from the behavior of the \emph{squared} modulus non-linearity, i.e., 
$$W[k]f :=|f\ast g_{k}|^2 = (f\ast g_{k}) (\overline{f\ast g_{k}}),$$
and note that $\widehat{W[k]f}$ is simply the auto-correlation of $\widehat{f}\,\widehat{g_k}$ (see the second row in  Figure \ref{fig10}). The squared modulus non-linearity therefore doubles the spectral support of $f\ast g_{k}$ and ``demodulates''  in the sense of the  spectrum $\widehat{W[k]f}$ being located symmetrically around the origin, both  irrespectively of the spectral location of $f\ast g_{k}$. The key observation is now that the modulus non-linearity behaves  similarly, as illustrated in Figure \ref{fig10}, third row.  In the following, we shall therefore allow ourselves to work with $\text{esupp}(\widehat{U[k]f}\,)\subseteq[-2R,2R]$, for all $k\in \mathbb{Z}\backslash\{ 0\}$. We hasten to add that this statement is based solely on numerical evidence   and we do not have a corresponding formal result. It is interesting to observe that the sigmoid, rectified linear unit, and hyperbolic tangent non-linearities, all exhibit very different behavior in this regard (see Figure \ref{fig10}, bottom row, for an illustration for the rectified linear unit). 
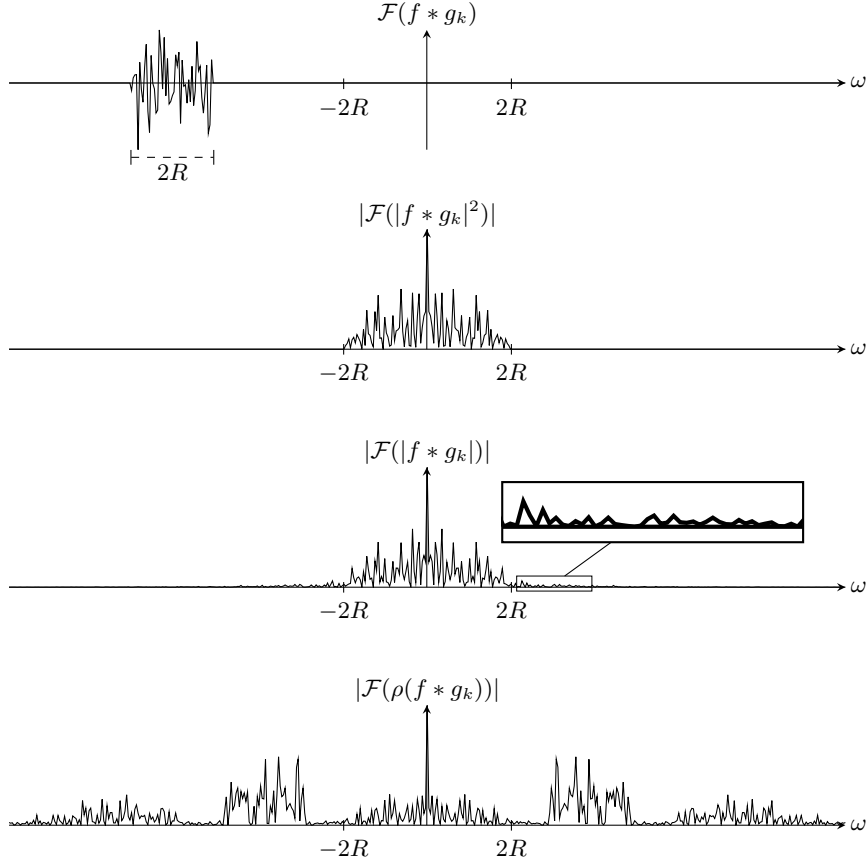
\begin{figure}
\centering

	\begin{tikzpicture}[scale=1]
	\tikzstyle{inti}=[draw=none,fill=none]

	\begin{axis}[width = 5in, height = 1.25in, axis lines=center, ticks=none]
    		\addplot +[mark=none,solid,black] table[x index=0,y index=1]{A1.dat};
    	\end{axis}
		\draw (4.45 cm,27pt) -- (4.45 cm,23pt) node[anchor=north] {\footnotesize$-2R$};
	\draw (6.68cm,27pt) -- (6.68 cm,23pt) node[anchor=north] {\footnotesize$2R$};
	
	\node[inti] at (2.175,-0.3) {\footnotesize $2R$};	
	\node[inti] at (5.5650,1.8cm) {\footnotesize $\mathcal{F}(f\ast g_{k})$};	
	\node[inti] at (11.3,.9cm) {\footnotesize $\omega$};
	
	\draw[dashed,|-|] (1.6150cm,-.1cm) -- (2.73cm,-.1cm);
	\end{tikzpicture}

	\begin{tikzpicture}[scale=1]
		\tikzstyle{inti}=[draw=none,fill=none]
	\begin{axis}[width = 5in, height = 1.25in, axis lines=center, ticks=none]
    		\addplot +[mark=none,solid,black] table[x index=0,y index=1]{A3.dat};
    	\end{axis}
	\draw (4.45 cm,2pt) -- (4.45 cm,-2pt) node[anchor=north] {\footnotesize$-2R$};
	\draw (6.68cm,2pt) -- (6.68 cm,-2pt) node[anchor=north] {\footnotesize$2R$};
		\node[inti] at (5.5650,1.8cm) {\footnotesize $|\mathcal{F}(|f\ast g_{k}|^2)|$};	
	
		\node[inti] at (11.3,.cm) {\footnotesize $\omega$};
	\end{tikzpicture}

\vspace{0.5cm}

	\begin{tikzpicture}[scale=1,spy using outlines=
	{rectangle, magnification=4, connect spies}]\vspace{0.5cm}
		\tikzstyle{inti}=[draw=none,fill=none]
	\begin{axis}[width = 5in, height = 1.25in, axis lines=center, ticks=none]
    		\addplot +[mark=none,solid,black] table[x index=0,y index=1]{A2.dat};
    	\end{axis}
		\draw (4.45 cm,2pt) -- (4.45 cm,-2pt) node[anchor=north] {\footnotesize$-2R$};
	\draw (6.68cm,2pt) -- (6.68 cm,-2pt) node[anchor=north] {\footnotesize$2R$};
		\node[inti] at (11.3,0cm) {\footnotesize $\omega$};
		\node[inti] at (5.5650,1.8cm) {\footnotesize $|\mathcal{F}(|f\ast g_{k}|)|$};
	\coordinate (spypoint) at (7.25,0.05);
  \coordinate (magnifyglass) at (8.558,1);
  \spy [black, width=4cm, height=.8cm] on (spypoint)
   in node[fill=white] at (magnifyglass);	
		
	\end{tikzpicture}
	
\vspace{0.5cm}
	\begin{tikzpicture}[scale=1]
		\tikzstyle{inti}=[draw=none,fill=none]
	\begin{axis}[width = 5in, height = 1.25in, axis lines=center, ticks=none]
    		\addplot +[mark=none,solid,black] table[x index=0,y index=1]{A4.dat};
    	\end{axis}
	\draw (4.45 cm,2pt) -- (4.45 cm,-2pt) node[anchor=north] {\footnotesize$-2R$};
	\draw (6.68cm,2pt) -- (6.68 cm,-2pt) node[anchor=north] {\footnotesize$2R$};
		\node[inti] at (5.5650,1.8cm) {\footnotesize $|\mathcal{F}(\rho(f\ast g_{k}))|$};	
	
		\node[inti] at (11.3,.cm) {\footnotesize $\omega$};
	\end{tikzpicture}

\caption{\small{Illustration of the demodulation and bandwidth doubling effect of the squared modulus (second row) and the modulus non-linearities (third row). The WH filters have prototype function $g$ with $\text{supp}(\widehat{g})=[-R,R]$. The rectified linear unit non-linearity (bottom row) is defined as $\rho(z):=\max\{0,\text{Re}(z)\}+\max\{0,\text{Im}(z)\}$, $z\in \mathbb{C}$.}}
\label{fig10}
\end{figure}
By induction over $n$, one can now  show that, for all $n\geq2$, we have
\begin{equation}\label{sdfnsdkjfhaaaa1111}
\Lambda^{n}_{\text{WH}, \, \text{sig}}
=\big\{(q,k) \ \big| \ q \in \Lambda^{n-1}_{\text{WH}, \, \text{sig}} \ \text{ and } |k|\leq \lfloor 3 - \delta R^{-1}\rfloor \big\} , \hspace{1cm} \text{if }2R,L>\delta, 
\end{equation}
and $\Lambda^{n}_{\text{WH}, \, \text{sig}}=0$, else, which implies
\begin{align}\label{sdlkfldsk4444f1222}
\Xi_{\text{WH}}(n,R,\delta,L)&=\begin{cases}
1, & \text{if } n=0,\\
2 \big\lfloor \frac{L-\delta}{R} + 1\big\rfloor,& \text{if } n=1 \text{ and } L>\delta,\\
2 \big\lfloor \frac{L-\delta}{R} + 1\big\rfloor\big\lfloor 3-\frac{\delta}{R}\big\rfloor^{n-1},& \text{if } n\geq2 \text{ and } 2R,L> \delta,\\
\end{cases}
\end{align}
and $\Xi_{\text{WH}}(n,R,\delta,L)=0$, else. We remark that  \eqref{sdlkfldsk4444f1222} follows from \eqref{sdfnsdkjfhaaaa1111} upon noting that,  from the second network layer onwards, either $U[(k_1,\dots,k_{n-1},-k_n)]f$ or $U[(k_1,\dots,k_{n-1},k_n)]f$ only contribute to $\Xi(n)$ (which, as explained above, only counts the number of \emph{distinct} feature maps). This follows from the symmetry relation  
\begin{align}
U[(k_1,\dots,k_{n-1},-k_n)]f &= U[-k_n]U[(k_1,\dots,k_{n-1})]f=\big|\underbrace{(U[(k_1,\dots,k_{n-1})]f)}_{\text{real-valued}} \ast g_{-k_n}\big|\nonumber \\
&=\big|(U[(k_1,\dots,k_{n-1})]f) \ast g_{k_n}\big|=  U[(k_1,\dots,k_n)]f, \quad \forall n\geq 2,\label{dfndkslflkdsnfff}
\end{align}
where the first equality in \eqref{dfndkslflkdsnfff} is by the following. 
\begin{lemma}\label{dsflsnlddddd}
Let $f\in L^2(\R)$ be real-valued and $g_{\lambda}$ either a WH or a wavelet filter as defined in Section \ref{Sec:prop}. Then, we have 
$$
U[-\lambda]f=|f\ast g_{-\lambda}| = |f\ast g_{\lambda}| =U[\lambda]f, \ \forall \lambda \in \Lambda. 
$$
\end{lemma}
\vspace{-0.3cm}
\begin{proof}The proof follows from basic Fourier calculus and by exploiting the symmetry property \eqref{sdlsdhkdjfshf}.
\end{proof} 
For the wavelet case, arguments similar to those leading to \eqref{sdfnsdkjfhaaaa1111} yield, for all $n\geq1$,
\begin{equation}\label{fsdklfjddsdfdfsdsf2}
\Lambda^{n}_{\text{wav}, \, \text{sig}}
=\big\{(q,j) \ \big| \ q \in \Lambda^{n-1}_{\text{wav}, \, \text{sig}} \ \text{ and } |j|\leq \lfloor \log_r(L^{(n-1)}) + 1\rfloor \big\} , \hspace{1cm} \text{if } L^{(n-1)}>1,
\end{equation}
and $\Lambda^{n}_{\text{wav}, \, \text{sig}}=\emptyset$, if $L^{(n-1)}\leq1$, where $L^{(n)}:=L(r^2-1)^n$, for $n\geq0$. This implies  
\begin{align}\label{slkndlnkddd}
\Xi_{\text{wav}}(n,r,L)=\begin{cases}
1, & \text{if } n=0,\\
2\lfloor \text{log}_r(L) + 1\big\rfloor,& \text{if } n=1 \text{ and } L>1,\\
\mathcal{O}\big(\text{log}^n_r(L) +2^n(n-1)!\big),& \text{if } n\geq 2 \text{ and } L>1 \text{ and } r> \sqrt{2},\\
2\lfloor \text{log}_r(L)+ 1 \big\rfloor^n, & \text{if } n\geq 2 \text{ and } L > 1 \text{ and } r=\sqrt{2},\\
\mathcal{O}\big(\text{log}^n_r(L)\big),& \text{if } M> n\geq 2 \text{ and } L>1 \text{ and } r< \sqrt{2},
\end{cases}
\end{align}
and $\Xi_{\text{wav}}(n,r,L)=0$, else, where $M:= 1+\text{log}_{r^2-1}(L).$  We can see that the parameter $r>1$ of the mother wavelet crucially impacts the index sets \eqref{fsdklfjddsdfdfsdsf2} and thereby the number of operationally significant nodes \eqref{slkndlnkddd}. Specifically, $r$  determines whether the effective bandwidths $L^{(n)}=L(r^2-1)^n$ of the feature maps increase, decrease, or remain constant as the layer index $n$ increases. For $r>\sqrt{2}$ we have bandwidth expansion, for $r<\sqrt{2}$   bandwidth contraction, and for $r=\sqrt{2}$ the effective bandwidths of the feature maps $U[q]f$ remain constant, i.e., $L^{(n)}=L$, for all $n\in \mathbb{N}$. 

\section{NETWORK TOPOLOGY INDUCED BY\\ OPERATIONALLY SIGNIFICANT NODES}
\label{dfsdf}
The scattering network architecture defined in Section \ref{architecture} has a tree topology with an infinite number of nodes per layer. The analysis in the previous section revealed, however, that the number $\Xi(n)$ of operationally significant nodes is finite in every network layer $n=0,\dots,N$. The goal of this section is to determine and characterize the network topology and the corresponding feature vector $\{ \Phi_\Omega^n(f)\}_{n=0}^{N}$ induced by the operationally significant nodes. For WH filters, we distinguish between the following  cases:
\begin{itemize}
\item[i)]{\emph{Shallow feature extraction}: If $L\leq\delta$ (i.e., the effective spectral support of the input signal $f$ is fully contained in the spectral gap $[-\delta,\delta]$, see Figure \ref{fig1234}, top row), then $\Xi_{\text{WH}}(0,R,\delta,L)=1$ and $\Xi_{\text{WH}}(n,R,\delta,L)=0$,  $n\geq1$. The feature vector $\{ \Phi_\Omega^n(f)\}_{n=0}^{N}$ consists of a single element, namely $f\ast \chi$, which is simply the output at the root of the network.}
\item[ii)]{\emph{Single-layer network}: If $L>\delta$ and $2R\leq\delta$ (i.e.,  the effective spectral support  of all feature maps is fully contained in the spectral gap $[-\delta,\delta]$), then $\Xi_{\text{WH}}(0,R,\delta,L)=1$, $\Xi_{\text{WH}}(1,R,\delta,L)= 2 \big\lfloor \frac{L-\delta}{R} + 1\big\rfloor$, and $\Xi_{\text{WH}}(n,R,\delta,L)=0$, $n\geq2$, which renders the network to have a single layer only. The corresponding  feature vector is given by $\{ \Phi_\Omega^n(f)\}_{n=0}^{N}=\{ f\ast \chi\}\cup\{|f\ast g_{k}|\ast \chi \}_{|k|\leq \lfloor (L-\delta)R^{-1}+1\rfloor}$.}
\item[iii)]{\emph{Constant-width network}: If $L>\delta$ and $R< \delta < 2R$ (i.e., only the spectral supports of the filters $g_k$, $k\in \{ -1,1\}$, overlap with  the interval $[-2R,2R]$), then the number of operationally significant nodes $\Xi_{\text{WH}}(n,R,\delta,L)=2 \big\lfloor \frac{L-\delta}{R} + 1\big\rfloor$, $n\geq 1$, is constant in $n$ (for $n=0$, we have $\Xi_{\text{WH}}(0,R,\delta,L)=1$). In this constant-width network, every network layer $n\geq 1$ contributes with  $2 \big\lfloor \frac{L-\delta}{R} + 1\big\rfloor$ elements to the feature vector.}
\item[iv)]{\emph{Expanding-width network}: If $L>\delta$ and $\delta\leq R$ (i.e., at least four filters $g_k$ overlap with the interval $[-2R,2R]$), then $\Xi_{\text{WH}}(n,R,\delta,L)=2 \big\lfloor \frac{L-\delta}{R} + 1\big\rfloor\big\lfloor 3-\frac{\delta}{R}\big\rfloor^{n-1}$, $n\geq 1$, which renders the network expanding width (for $n=0$, we have $\Xi_{\text{WH}}(0,R,\delta,L)=1$).}
\end{itemize}
We note that for $L>\delta$, it is the bandwidth $R$ of the WH prototype function $g$ that determines  the transition between the network topologies above, see Figure \ref{fig12344444}. 

\begin{figure}
\begin{center}

\begin{tikzpicture}
	\begin{scope}[scale=1.1]
	\draw[->] (0,0) -- (10,0) node[right] {$R$};
	\draw[->] (0,-.5) -- (0,.5) node[above] {};
%
		\draw (2.5 cm,2pt) -- (2.5 cm,-3pt) node[anchor=north] {\footnotesize$\frac{\delta}{2}$};
		\draw (5 cm,2pt) -- (5 cm,-3pt) node[anchor=north] {\footnotesize$\delta$};


		\draw (5 cm,.25cm) node[anchor=north] {\footnotesize$[$};
		\draw (2.5 cm,.25cm) node[anchor=north] {\footnotesize$]$};
		\draw (1.25 cm,.5cm) node[anchor=north] {\footnotesize$\text{single-layer}$};
		\draw (3.75 cm,.5cm) node[anchor=north] {\footnotesize$\text{constant-width}$};
		\draw (7.25 cm,.5cm) node[anchor=north] {\footnotesize$\text{expanding-width}$};

			

               
	\end{scope}
\end{tikzpicture}

\end{center}
\caption{\small{Transition between network topologies (as induced by operationally significant nodes) as a function of $R$.}}
\label{fig12344444}
\end{figure}
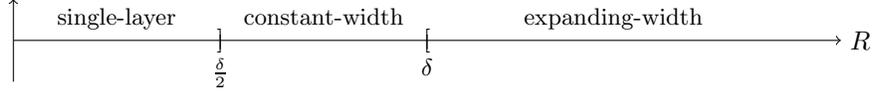
We next turn to wavelet filters with the following cases of interest:
\begin{itemize}
\item[i)]{\emph{Depth-pruned network}: If $L>\max\{1,r\}$, $r<\sqrt{2}$, and $N> M=1+\text{log}_{r^2-1}(L)$ (i.e., the effective bandwidths $L^{(n)}=L(r^2-1)^n$ of the feature maps are decreasing in $n$ and are eventually smaller than $1$ and hence contained in the spectral gap $[-1,1]$), then we have $\{ \Phi_\Omega^n(f)\}_{n=0}^{N}=\{ \Phi_\Omega^n(f)\}_{n=0}^{M}$. This means that from the $M$-th  layer onwards, there are no more non-zero signals to be propagated to deeper layers.}
\item[ii)]{\emph{Extremely-narrow network}: If $1<L\leq r= \sqrt{2}$  (i.e., the effective bandwidths $L^{(n)}=L$ of the feature maps are constant in $n$, with $n\geq1$, and overlap with the spectral supports of $g_k$, $k\in \{-1,1\}$, only), then the number of operationally significant nodes $\Xi_{\text{wav}}(n,r,L)=2$,  $n\geq 1$, is constant in $n$ (for $n=0$, we have $\Xi_{\text{wav}}(0,r,L)=1$). Every network layer $n\geq 1$ contributes with two elements to the feature vector. }
\end{itemize}

\section{MINIMIZING THE AVERAGE NUMBER OF \\ OPERATIONALLY SIGNIFICANT NODES PER LAYER}

The purpose of this section is to analyze the impact of   topology reduction on the average number of operationally significant nodes per layer. For simplicity of exposition, throughout this section, we focus on the WH case. We take the parameters $N$, $\delta$, and $L$ to be fixed and  assume  i) that the effective bandwidth $L$ of the input signal satisfies  $L>\delta$ (which guarantees that we are not in the (trivial) shallow feature extraction situation, see Section \ref{dfsdf}) and  ii) that the network depth satisfies $N\geq 3$.

We first recall that, thanks to \eqref{432897498237}, the (exponential) decay factor $a$ can be tuned through the parameter  $R$. Specifically, reducing the bandwidth $R$ of the WH prototype function $g$ implies faster (guaranteed) energy decay (see also Figure \ref{fig6}). Increasing $R$ implies slower (guaranteed) energy decay with $R$ eventually violating the condition $R\leq2\delta$ needed for validity of the statement in Theorem \ref{thm4}. In the following, we determine the optimal value $R^\ast$ in the exponential-decay regime $R\in (0,2\delta)$ of $W_n(f)$ that minimizes the average number of operationally significant nodes per layer  given by  
\begin{equation}\label{slkdnskj}
\Theta_{\text{WH}}(N,R,\delta,L) := \frac{1}{N}\sum_{n=1}^{N}\Xi_{\text{WH}}(n,R,\delta,L).\end{equation}
In order to minimize the expression in \eqref{slkdnskj} over the interval $(0,2\delta)$, we distinguish between three cases:
\begin{itemize}
\item[i)]{If $R\in [\delta,2\delta)$, then we are in the situation of an expanding-width network, and we have \begin{align}\hspace{-0.5cm}\Theta_{\text{WH}}(N,R,\delta,L)&=2 \bigg\lfloor \frac{L-\delta}{R} + 1\bigg\rfloor \frac{1}{N}\sum_{n=0}^{N-1}\bigg\lfloor \underbrace{3-\frac{\delta}{R}}_{\in [2,\frac{5}{2})}\bigg\rfloor^{n}=2 \bigg\lfloor \frac{L-\delta}{R} + 1\bigg\rfloor \frac{1}{N}\sum_{n=0}^{N-1}2^n\nonumber\\
&=2 \bigg\lfloor \frac{L-\delta}{R} + 1\bigg\rfloor \frac{1}{N}(2^N-1).\label{fdslknflnfd}
\end{align} }
\item[ii)]{For $R\in (\frac{\delta}{2},\delta)$, we have a constant-width network and 
\begin{equation}\label{slkfnlknfd}\Theta_{\text{WH}}(N,R,\delta,L)=2 \bigg\lfloor \frac{L-\delta}{R} + 1\bigg\rfloor \frac{1}{N}\sum_{n=0}^{N-1}\bigg\lfloor \underbrace{3-\frac{\delta}{R}}_{\in (1,2)}\bigg\rfloor^{n}=2 \bigg\lfloor \frac{L-\delta}{R} + 1\bigg\rfloor. 
\end{equation}
} 
\item[iii)]{If $R\in (0,\frac{\delta}{2}]$, we get a single-layer network, and $\Theta_{\text{WH}}(N,R,\delta,L)=2 \big\lfloor \frac{L-\delta}{R} + 1\big\rfloor.$
} 
\end{itemize}
Next, we note that the function $R\mapsto \big\lfloor \frac{L-\delta}{R} + 1\big\rfloor$, $R\in (0,2\delta)$, is monotonically decreasing in $R$, which allows us to conclude that, owing to  ii) and iii), $ R^\ast \notin (0,\frac{\delta}{2}]$. Moreover, thanks to the monotonicity of the mapping $R\mapsto \big\lfloor \frac{L-\delta}{R} + 1\big\rfloor$, $R\in (0,2\delta)$, it is sufficient to evaluate the expression \eqref{fdslknflnfd} for $R=2\delta$ and \eqref{slkfnlknfd} for  $R=\delta$ and to determine which of the resulting two values is smaller. Specifically, we have
\begin{align} \Theta_{\text{WH}}(N,2\delta,\delta,L)&=2 \bigg\lfloor \frac{L}{2\delta} + \frac{1}{2}\bigg\rfloor \frac{1}{N}(2^N-1)=2 \bigg\lfloor \frac{1}{2}\bigg( \frac{L}{\delta} + 1\bigg)\bigg\rfloor \frac{1}{N}(2^N-1)\geq  \bigg\lfloor  \frac{L}{\delta} \bigg\rfloor \frac{1}{N}(2^N-1)\label{sdlfndsfnlkdsn}\\
&> 2 \bigg\lfloor \frac{L}{\delta} \bigg\rfloor=\Theta_{\text{WH}}(N,\delta,\delta,L),\label{sdlfndsfnlkdsn1}
\end{align} 
where in \eqref{sdlfndsfnlkdsn} we used $2\lfloor \frac{x+1}{2}\rfloor \geq \lfloor x \rfloor$, $x\geq 0$, and \eqref{sdlfndsfnlkdsn1} is thanks to  $N\geq 3$, which, in turn, is  by assumption. This implies $R^\ast \in (\frac{\delta}{2},\delta)$ and renders the network constant-width.

\bibliography{report2}   
\bibliographystyle{spiebib}   

\end{document}